\documentclass[11pt,english]{article}

\usepackage[utf8]{inputenc} 
\usepackage[T1]{fontenc}    
\usepackage{hyperref}       
\usepackage{url}            
\usepackage{booktabs}        
\usepackage{amsfonts}       
\usepackage{nicefrac}       
\usepackage{microtype}      
\usepackage{dirtytalk}      
\usepackage{caption}
\usepackage{subcaption}
\usepackage[linesnumbered,ruled]{algorithm2e}
\usepackage{amsthm}
\usepackage{amsmath}
\usepackage{comment}
\usepackage{bbm}
\usepackage{csvsimple}
\usepackage{authblk}
\DeclareMathOperator*{\argmin}{argmin}

\newcommand{\floor}[1]{\lfloor #1 \rfloor}
\newcommand{\openr}{\hbox{${\rm I\kern-.2em R}$}}
\newcommand{\pl}{\parallel}

\usepackage{graphicx}

\newtheorem{theorem}{Theorem}[section]

\newtheorem{lemma}[theorem]{Lemma}
\newtheorem{definition}{Definition}

\title{Conditional Super Learner}

\begin{document}

\author{Gilmer Valdes\textsuperscript{a}\thanks{CONTACT Gilmer Valdes. Email: Gilmer.Valdes@ucsf.edu} and Yannet Interian\textsuperscript{b} and Efstathios Gennatas\textsuperscript{a} and Mark Van der Laan \textsuperscript{c} \\
\textsuperscript{a}Department of Radiation Oncology, University of California San Francisco, San Francisco, CA; \textsuperscript{b}Master of Science in Data Science, University of San Francisco, San Francisco, CA; \textsuperscript{c} Division of Biostatistics, University of California, Berkeley, CA
}

\maketitle 

\begin{abstract}
In this article we consider the Conditional Super Learner (CSL) algorithm which selects the best model candidate from a library conditional on the covariates. The CSL expands the idea of using cross validation to select the best model and merges it with meta learning. Here we propose a specific algorithm that finds a local minimum to the problem posed, proof that it converges at a rate faster than $O_p(n^{-1/4})$ and offer extensive empirical evidence that it is an excellent candidate to substitute stacking or for the analysis of Hierarchical problems. 
\end{abstract}


\section{Introduction}

The idea of combining different models to obtain one that is better than any of its constituents (meta learning) has been explored extensively and it is currently used in many applications \cite{breiman1996stacked, leblanc1996combining, wolpert1992stacked}. Meta learning today, however, mainly consists of creating linear combinations of models (i.e stacking). Its purpose is to improve the accuracy of the individual models, albeit at the expense of interpretability. Related ideas are also explored for ensemble methods which create linear combination of simpler models. Two main ensemble methods can be highlighted: bagging and boosting \cite{breiman1996bagging,freund1996experiments}. In bagging, models are averaged to reduce the variance of individual models and improve accuracy. In boosting, simple models are sequentially learned reducing the bias of the estimator at each step \cite{friedman2001greedy}. 

Usually thought independently from meta learning, the use of cross validation to select the best algorithm from a library (either different models or different hyperparameters) is widely popular \cite{efron1983estimating}. Establishing the theoretical basis for designing an oracle algorithm that will select the best from a library of models (using cross validation), Van der Laan et al demonstrated that cross validation can be used more aggressively than previously thought, terming the cross validation selector ``super learner" \cite{van2007super}. Specifically, it was shown that if the number of candidate estimators, $K(n)$, is polynomial in sample size, then the cross validation selector is asymptotically equivalent to the oracle selector --one that knows the best algorithm \cite{van2007super}. Similarly to the empirical use of cross validation, the super learner proposes to select one model from a library for all the observations. However, for complex functions and simple models in the library (e.g to afford interpretability for instance) it is possible that the model selected to be the best in one region of the covariates might not be the best in another. 

In the present article,  we expand on this idea and investigate an algorithm that selects the best model from a library conditional on the covariates, called here Conditional Super Learner (CSL). This meta algorithm can be thought as learning in the cross validation space. With the CSL, therefore, we investigate a meta learning strategy that instead of forming a linear combination of models, it reduces the bias of the models in the library by selecting them conditional on the covariates. We show how the CSL has implications for both the accuracy of models and their interpretability. Specifically, in this article we:

\begin{enumerate}

\item Develop the theoretical foundations for the \emph{Conditional Super Learner}: An algorithm that selects the best model from a library conditional on the covariates.

\item Illustrate how the CSL is a generalized  partitioning algorithm that finds different boundary functions (not just vertical cuts as CART does) with  $\mathcal{M}$-estimators algorithms at the nodes.

\item Establish the connection between CSL and interpretability. 

\item Show empirically how CSL improves over the regular strategy of using cross validation to select the best model for all observations. 

\item Show empirically how CSL can give better $R^2$ than stacking in a set of regression problems.

\item Show empirically how CSL performs in the analysis of Hierarchical Data.

\end{enumerate}

\begin{figure}[t!]
 \includegraphics[width=0.5\textwidth]{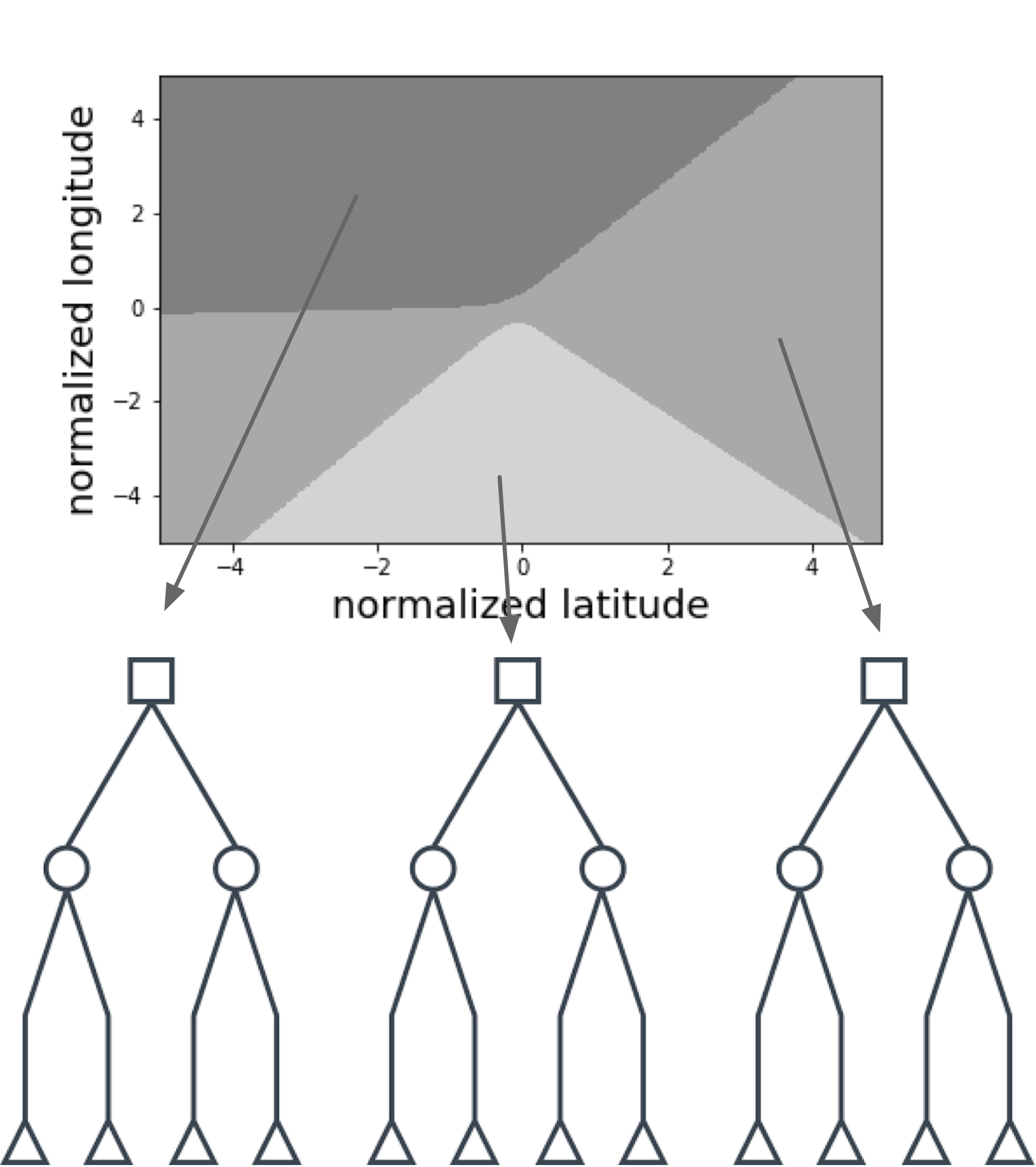}
\centering
\caption{This diagram shows an application of the CLS model. In this dataset we have 4 variables: number of bedrooms, bathrooms, latitude and longitude to predict house prices. The rectangular region shows how the oracle divides the latitude and longitude (normalized) in 3 regions. Each region has its own expert (using number of bedrooms and bathrooms), represented here by a diagram of a tree, to makes predictions.}
\label{Figure1}
\end{figure}
 
\section{Conditional Super Learner}

The algorithms that we discuss in this paper are supervised learning algorithms. The data are a finite set of paired observations $\mathcal{X}  = \{(x_i,y_i)\}_{1}^{N}$. The input vectors $x$, whose components are called here covariates, are assumed to be on $\mathbb{R}^p$, while $y$ can be either a regression or classification label. 

We propose to solve supervised learning problems by 1) dividing the input space into a set of regions that are learned by an iterative algorithm and  2) fit simple interpretable models that we call ``experts" to the data that fall in these regions. The regions are learned by fitting a multi-class classification model that we call the ``oracle" which learns which expert should be used on each region. Given an oracle $o(x)$, region $k$ is defined as $\{o(x)=k\}$, that is, the set of points for with the oracle predict to use the function $f_k(x)$. 

An example of an application of Conditional Super Learner (CSL) is shown in Figure \ref{Figure1}. In here we have 4 variables to predict houses prices: bedrooms,	bathrooms, latitude, longitude. The rectangular region shows how the oracle divides the latitude and longitude (normalized) in 3 regions.
Each region has its own expert, represented here by a diagram of a tree, to make predictions. Each of the experts has as input the 4 variables. 

As with any meta-learning algorithm, the Conditional Super Learner algorithm for learning the oracle $o(x)$ (given the fits of the $K$ experts) will be applied to a cross-validated data set, using $V$-fold cross validation. That is, for each $Y_i$ falling in one of the $V$ validation samples, we have a corresponding training sample. We  couple each observation $Y_i$  with $K$ expert algorithms trained on subsets (from current best estimate of oracle) of its corresponding training data set, thereby creating a cross-validated data set of $N$ observations. In this section, for the sake of explaining the conditional super-learner algorithm, this formality of cross validation will be suppressed, but in our theoretical section we make the full conditional super-learning algorithm formal. 

\subsection{Definition of CSL}

Given an oracle $o(x)$ and $K$ experts models $\{F_{k}(x)\}_{k=1}^{K}$ fitted on each of the corresponding regions $\{o(x)=k\}$, the CSL can be defined as:

\begin{equation}
\label{eqn1}
CSL(x) = \sum_{k=1}^{K} \mathbbm{1}\{o(x) = k\}F_k(x)
\end{equation}

where $o(x) \in \{1,2,..K\}$. $CSL(x)$ is the Conditional Super Learner that outputs the best model $F_{k}(x)$ from a library of $K$ models conditional on the covariate $x$. The idea is to find the $o(x)$ and  corresponding fits $\{F_{k}(x)\}_1^K$ that minimize a given loss function over the training data:

\begin{equation}
\label{eqn2}
\argmin_{o,\{F_{k}\}_1^K} \sum_{i=1}^{N} L\big(y_i,\sum_{k=1}^{K} \mathbbm{1}\{o(x_i)=k\} F_{k}(x_i)\big)
\end{equation}

\subsection{Fitting the oracle}

To find the solution to equation \ref{eqn2} we will employ a trick often used in machine learning. We will iterate between solving $o(x)$ and solving for $\{F_{k}(x)\}_1^K$. To solve for $o(x)$ we will assume that all $\{F_{k}(x)\}_1^K$ are known, the library, and that we also have  unbiased estimations (i.e., cross-validated) of the loss at each training point $L(y_i, F_k(x_i)) $. In this case, $CSL(x)$ will aim to find the best $o(x)$ that minimizes the loss function over the training data

\begin{equation}
\label{eqn3}
\argmin_{o(x)} \sum_{i=1}^{N} L\left(y_i ,\sum_{k=1}^{K} \mathbbm{1}\{o(x_i)=k\} F_{k}\left(x_i\right)\right)
\end{equation}

and using the definition of the indicator function, we can take the sum outside of the loss function and get Equation \ref{eqn4}: 

\begin{equation}
\label{eqn4}
\argmin_{o(x)} \sum_{i=1}^{N} \sum_{k=1}^{K} \mathbbm{1}\{o(x_i)=k\} L(y_i , F_{k}(x_i))  
\end{equation}

To introduce how we fit the oracle, we define a new dataset that we called ``extended" dataset. This is a dataset for a multi-class classification problem with $K$ classes, each class corresponding to one of the expert models. This dataset has $K\cdot N$ observations -- each $x_i$ appears $K$ times with corresponding labels $\{1, \dots K\}$ and a specific weight.\\

\begin{definition}
Given dataset  $\mathcal{X} = \{(x_i,y_i)\}_{1}^{N}$, expert functions $\mathcal{F}=\{F_{k}\}_1^K$ we define the {\bf extended dataset} $\mathcal{X} _{\mathcal{F}} = \{(\tilde{x}_i,z_i,w_i)\}_{1}^{K\cdot N}$ with $K\cdot N$ observations where:
\begin{itemize}
    \item $\tilde{x}_i = x_{\floor{i/K} + 1}$   
    \item $z_i = i \mod K + 1$  
    \item $w_i$ is a weight on observation $(\tilde{x}_i,z_i)$ and is defined in the following way: 
\end{itemize}

Let $l_i$ be $K$ dimensional vector $l_i =(L(y_i, F_{1}(x_i)), \dots,L(y_i, F_{K}(x_i))$ where element $k$ is the loss of expert $k$ at point $(x_i, y_i)$. Let  $ONE_K$ a $K \times K$ matrix of all ones and $DIAG_K$  a $K \times K$ matrix with ones in the diagonal and zeros everywhere else. 

\begin{equation}
(w_{iK +1}, \dots w_{iK + K})^T =  [ONE_K - DIAG_K]^{-1} l_i^T
\label{eqn5}
\end{equation}
\end{definition}

\begin{lemma}
Solving problem \ref{eqn4} is equivalent to finding the oracle $o(x)$ that minimizes the weighted miss classification error of a extended dataset $\mathcal{X} _{\mathcal{F}}$:

\begin{equation}
\argmin_{o(x)}  \sum_{i=1}^{N\cdot K} w_i \mathbbm{1}\{o(\tilde{x_i}) \neq z_i \} 
\label{eqn6}
\end{equation}
\label{lemma1}
\end{lemma}

\begin{proof}

First note that a missclassification loss for a multi-class classification problem can be written as $L(y, f(x)) = \mathbbm{1}\{f(x) \neq y\}$. We want to write Equation \ref{eqn4} as a missclassification loss of a classification problem. For each observation $(x_i, y_i)$ consider the weighted dataset $\{(x_i, 1, w_1), (x_i, 2, w_2), \dots (x_i, K, w_K) \}$ with missclassification loss $\sum_{k=1}^{K} w_k \mathbbm{1}\{o(x_i) \neq k\}$. That is, we want to find for each observation $(x_i, y_i)$ weights $(w_1, \dots, w_K)$ such that:

\begin{equation}
\sum_{k=1}^{K} \mathbbm{1}\{o(x_i)=k\} L(y_i , F_{k}(x_i))  = \sum_{k=1}^{K} w_k \mathbbm{1}\{o(x_i) \neq k\}
\label{eqn7}
\end{equation}

Since $o(x_i)$ can just have values in $\{1, \dots, K\}$ we can consider all the options. For example, if $o(x_i) = k$, the equality in Equation \ref{eqn7} becomes
$L(y_i, F_k(x_i)) = \sum_{j=1}^K w_j - w_k$. If we consider all possible values for $o(x_i)$ we get the following set of equations:

\begin{align*}
L(y_i, F_1(x_i)) & = \sum_{j=1}^K w_j - w_1\\
&\vdots\\
L(y_i, F_K(x_i)) & = \sum_{j=1}^K w_j - w_K
\end{align*}

The previous equation can be written in matrix form $l_i^T = [ONE_K - DIAG_K] (w_1, \dots w_K)^T$. Which gives us:

$$(w_1, \dots, w_K)^T= [ONE_K - DIAG_K]^{-1} l_i^T$$

\end{proof}

As a result of Lemma \ref{lemma1}, $o(x)$ is the solution of a multi-class classification problem on the extended dataset  $\mathcal{X} _{\mathcal{F}}$. We approximate o(x) by fitting any standard classification algorithm on $\mathcal{X} _{\mathcal{F}}$.

\subsection{Fitting the experts}
 
 Similarly to the previous section, in order to fit the experts we assume that $o(x)$ is known. Then Equation \ref{eqn2} becomes $K$ independent classification/ regression problems that minimize the empirical loss over observations $\{i:o(i)=k\}$, for each $k=1,\ldots,K$, which is generally already  solved by standard machine learning algorithms.

\begin{equation}
\label{eqn8}
\argmin_{F_k} \sum_{x_i:o(x_i) = k}  L(y_i ,F_k(x_i))
\end{equation}

\subsection{A two step algorithm}
 Finding  $\{F_k(x)\}_1^K$ indicates that  equation \ref{eqn2} can be minimized iteratively. Following the K-mean algorithm's philosophy, let us propose the minimization of equation \ref{eqn2} in two steps: one to fit the oracle and the other to fit the experts. Please note that if we take into consideration that at every time that  each step is applied the Loss function  decreases, the convergence to a local minimum is guaranteed. Of course, this is only true if we use, for each observation, the estimation of $(y_i ,F_k(\bf{x}_i)$ on the training data. This, however, will most likely result in overfitting. After this discussion we are ready to write  Conditional Super Learner pseudo code (see above). 

\begin{algorithm}[t]
\SetAlgoLined
\LinesNotNumbered
\KwIn{
    $ \mathcal{X} = \{ (x_i, y_i)\}_1^N$; $\mathcal{F} = (F_1, F_2, \dots, F_K)$ 
}
Initialize: for each sample split $v=1,\ldots,V$, fit the experts $\mathcal{F} = (F_1, F_2, \dots, F_K)$  on initial subsets of the $v$-th training data set. For each $i$, let $F_{k,-i}$ be the $k$-th expert trained on training sample that excludes $Y_i$. Construct the corresponding cross-validated data set $(Y_i,F_{1,-i}(x_i),\ldots,F_{K,-i}(x_i))$, $i=1,\ldots,N$. 

\For{$t=1$ : iterations}{
 For each point and each expert compute: $L(y_i ,F_{k,-i}(x_i))$\\
 Create extended dataset $\mathcal{X}_{\mathcal{F}}$\\
 Fit $o(x)$ on $\mathcal{X}_{\mathcal{F}}$\\
 Re-fit each expert $F_k$ on $\{o(x)=k\}$ for the $V$-training samples. \\
}
Based on final $o(x)$, refit each expert $F_k$ on $\{o(x)=k\}$ for total sample.
\KwResult{$\sum_{k=1}^{K} \mathbbm{1}\{o(x)=k\}F_{k}(x)$}
\caption{{Conditional Super Learner \textbf{(CSL)}}}
\label{algo:duplicate}
\end{algorithm}

\section{Theoretical Guarantees}

In this section we will formalize meta-learning in general and the Conditional Super Learner in particular. Specifically we will: 

\begin{enumerate}

\item Formalize meta-learning, cross validation and the Conditional Super Learner.

\item Prove a rate of convergence theorem  for a general Super Learner that uses cross validation .

\item Highlight the bias variance trade-off for different meta-learning algorithms.

\item Give practical recommendations on how to best use aggressive meta-learning algorithms. 

\item Connect meta-learning with a double super-learner.

\end{enumerate}

Due to the complexity of the notation this section will be self contained. The readers interested in a more practical use of the CSL can skip this section.

\subsection{Formalizing meta learning}

Let us start by specifying the data generating process, the type of candidate estimators considered, the loss function, cross validation and finally meta-learning. This section contains all necessary assumptions and definitions before we prove formal results for the performance of the Super Learner, in particular, w.r.t. its oracle choice.   

\subsubsection{Data Generating Process}
Let $O_1,\ldots,O_n$ be $n$ i.i.d. copies of a random variable $O\sim P_0$ with data distribution $P_0$ that is known to be an element of a specified statistical model ${\cal M}$. In a regression or classification application we have $O=(X,Y)$ for a covariate vector $X$ 	and outcome  $Y$, where $Y$ could be categorial, ordered discrete, or continuous. Let $O\in [0,\tau_o]\subset \openr^{d_1}$ be a Euclidean valued random variable with realizations in a cube $[0,\tau_o]$. Let $P_n$ be the empirical probability measure of $O_1,\ldots,O_n$. Let $\Psi:{\cal M}\rightarrow {\bf \Psi}=\Psi({\cal M})$ be a particular functional parameter of $P_0$.  In our setting $\Psi(P_0)$ could be a  conditional mean $E_0(Y\mid X)$ (regression) or a conditional probability distribution $P(Y=\cdot\mid X)$ (classification).

\subsubsection{Restriction on the type of functional parameters considered}
Each possible value  $\psi$ of this functional parameter is a $d$-variate real valued function $x\rightarrow\psi(x)$, where we assume that $\psi:[0,\tau]\subset\openr^d\rightarrow \openr$ is defined on a $d$-dimensional cube $[0,\tau]$. We  assume that the parameter space ${\bf \Psi}\subset D[0,\tau]$ is contained in the space of multivariate real valued cadlag functions that are assumed to be right-continuous with left-hand limits \cite{neuhaus1971weak}. Another possible assumption we emphasize is that the so called sectional variation norm of each $\psi\in {\bf \Psi} $ is bounded \cite{gill1995inefficient, van2017generally}.

To define the sectional variation norm of a function $\psi\in D[0,\tau]$, for each subset $s\subset\{1,\ldots,d\}$, we define its $s$-specific section $\psi_s(x)=\psi(x_s,0_{-s})$, where for a vector $x$ $x_s=(x(j):j\in s)$ and $x_{-s}=(x(j):j\not \in s)$.
For each section $\psi_s$ we can compute its variation norm as $\int_{(0_s,\tau_s]}\mid  \psi_s(du)\mid$ which can be represented as a limit over a partitioning in $\mid s\mid$-dimensional cubes (left-open, right-closed) of the sum over the cubes in the partitioning of the absolute value of the measure $\psi_s$ assigns to this cube. Here we are reminded that the measure $f((a,b])$ a function $f$ assigns to a cube $(a,b]$ is defined as a generalized difference over the $2^{\mid s\mid}$ corners of $(a,b]$ where $\mid s\mid$ denotes the size of the set $s$: for example, for $d=2$ $f(((a_1,a_2),(b_1,b_2)])=f(b_1,b_2)-f(a_1,b_2)-f(b_1,a_2)+f(a_1,a_2)$.
We write $\pl \psi_s\pl_v =\int_{(0_s,\tau_s]}\mid \psi_s(du)\mid$. 
The sectional variation norm of $\psi$ is defined as the sum over all the sections of the section-specific variation norm:
\[
\parallel \psi\parallel_v^*=\mid \psi(0)\mid+\sum_{s\subset\{1,\ldots,d\}} \int_{(0_s,\tau_s]} \mid \psi_s(du)\mid .\]
We also note that for any cadlag function $\psi$ with finite sectional variation norm $\pl \psi\pl_v^*$ we have the following representation \cite{van2006oracle, van2017generally}:
\begin{eqnarray*}
\psi(x)&=&\psi(0)+\sum_{s\subset\{1,\ldots,d\}} \int_{(0_s,x_s]} \psi_s(du) \\
&=& \psi(0)+\sum_{s\subset\{1,\ldots,d\}} \int \phi_{s,u}(x) \psi_s(du),\end{eqnarray*}
where $\phi_{s,u}(x)=I(x_s\geq u)$ is the tensor product of indicators $I(x_j\geq u_j)$ over $j\in s$ defined by knot point $u=(u(j):j\in s)$.
For simplicity, we will assume that there is a universal bound $\sup_{\psi\in {\bf \Psi}}\pl \psi\pl_v^*<\infty$ on the sectional variation norm over the parameter space, even though this assumption can be prevented by considering a sieve ${\bf \Psi}_{c_j}=\{\psi\in {\bf \Psi}:\pl \psi\pl_v^*<c_j\}$ where the variation norm bound $c_j$ is an increasing sequence converging to infinity, and selecting $j$ with cross validation.

\subsubsection{Loss Function}

Our goal is estimation of $\Psi(P_0)$.  Let $L(\psi)(O)$ be a loss function for $\psi_0=\Psi(P_0)$ so that
$\psi_0=\arg\min_{\psi\in {\bf \Psi}}P_0 L(\psi)$ minimizes the expectation of the loss (i.e., risk), where we use notation $Pf\equiv \int f(o)dP(o)$. Assume $M_1\equiv \sup_{o,\psi}\mid L(\psi)-L(\psi_0)(o)\mid $ and $M_2=\sup_{\psi}P_0\{L(\psi)-L(\psi_0)\}^2/P_0(L(\psi)-L(\psi_0))$ are both finite, where the suprema in these definitions are over $\psi\in \Psi({\cal M})$ and over  $o$  in a support of $P_0$.
The bounds $M_1,M_2$ guarantee that the cross validation selector is well behaved and satisfies oracle inequalities \cite{van2003unified, van2006oracle,van2007super}

We also assume that $L({\bf \Psi})=\{L(\psi):\psi\in {\bf \Psi}\}$ is contained in set of cadlag functions $D[0,\tau_o]$ and that $\sup_{\psi\in {\bf \Psi}} \pl L(\psi)\pl_v^*<\infty$.
In other words, for any $\psi$ $L(\psi)$ is contained in class of cadlag functions of $O$ with a universal bound on the sectional variation norm. This is a weak regularity assumption on the loss function $L()$ since ${\bf \Psi}$ is itself also contained in the class of cadlag functions with a universal bound on the sectional variation norm. The loss-function also implies a loss-based dissimilarity 
\[
d_0(\psi,\psi_0)\equiv P_0L(\psi)-P_0L(\psi_0),\]
which will behave as a square of an $L^2(P_0)$-norm, due to $M_2<\infty$.

\subsubsection{Library of candidate estimators, cross validation , oracle selector and meta-learning}

An estimator can be represented as a mapping from the empirical measure $P_n$ into the parameter space. We let ${\cal M}_{np}$ be the set of al possible realizations of an empirical probability measure, including its limits as $n$ converges to infinity (i.e., a nonparametric statistical model), so that an estimator $\hat{\Psi}:{\cal M}_{np}\rightarrow {\bf \Psi}$.
We start out with a set $\hat{\Psi}_j:{\cal M}_{NP}\rightarrow {\bf \Psi}$ of $J$ candidate estimators, $j=1,\ldots,J$.

{\bf $V$-fold cross validation:} We define a $V$-fold cross validation scheme that maps an empirical probability measure $P_n$ into  an empirical measure of a training sample $P_{n,v}$ and corresponding empirical measure of the validation sample $P_{n,v}^1$,  across $v$-specific sample splits, $v=1,\ldots,V$. 

{\bf Meta-learning model:}
Let ${\cal F}$ be a set of functions $f:\openr^J\rightarrow \openr$, so that  $\psi_{f,n}=\hat{\Psi}_f(P_n)\equiv f(\hat{\Psi}_j(P_n):j=1,\ldots,J)$ represents a candidate estimator of $\psi_0$ that combines the $J$-candidate estimators, where, as a function of $x$, it is evaluated as $\psi_{f,n}(x)=f(\psi_{jn}(x):j=1,\ldots,J)$. One would restrict $f$ to be such that $f(\psi_1,\ldots,\psi_J)\in {\bf \Psi}$ for any $(\psi_1,\ldots,\psi_J)\in {\bf \Psi}^J$. This set of functions ${\cal F}$ defines a class of candidate estimators $(\hat{\Psi}_f: f\in {\cal F})$, where each choice is a particular ensemble from the $J$  estimators in the library.

{\bf Cross validated risk for candidate estimator $\hat{\Psi}_f$:}
For any $f$-specific candidate estimator $\hat{\Psi}_f$ we can evaluate its performance by its cross-validated risk:
\[
CV(\hat{\Psi}_f,P_n)=\frac{1}{V}\sum_{v=1}^V P_{n,v}^1 L(\hat{\Psi}_f(P_{n,v})),\]
i.e. the average across sample splits of the empirical mean over validation sample of the losses of the candidate estimator based on training sample.

{\bf{${\cal F}$-cross validation selector:}}
Let
\[
f_n=\arg\min_{f\in {\cal F}} CV(\hat{\Psi}_f,P_n)\]
be the Super Learner defined by minimizing the cross-validated risk of $\hat{\Psi}_f$ over all $f\in {\cal F}$.
A simple variation of this is that $f_n$ is only an approximation of this minimum obtained by using some kind of iterative greedy algorithm.
We also refer to $f_n$ as the ${\cal F}$-specific cross validation selector.

{\bf ${\cal F}$-super-learner:}
The ${\cal F}$-super-learner $\hat{\Psi}_{\cal F}(P_n)$ is now defined by the corresponding estimator applied to the whole sample:
\[
\hat{\Psi}_{\cal F}(P_n)\equiv \hat{\Psi}_{f_n}(P_n),\]
or, the immediate available (not requiring rerunning $\hat{\Psi}_{f_n}$ on whole sample after having determined the cross validation selector $f_n$): 
\[
\hat{\Psi}_{\cal F}(P_n)=\frac{1}{V}\sum_{v=1}^V \hat{\Psi}_{f_n}(P_{n,v}).\]
The latter choice assumes that ${\bf \Psi}$ is convex so that a simple average represents indeed a sensible estimator (and, assuming the loss-function is convex, it will satisfy the same oracle inequalities as presented below).  For the sake of theoretical analysis of the ${\cal F}$-specific super-learner, we assume that $f_n$ is the actual minimizer over ${\cal F}$ of the cross-validated risk .

{\bf ${\cal F}$-oracle selector:}
The cross validation selector $f_n$ is aiming to estimate the oracle selector
\begin{eqnarray*}
f_{0n}&=&\arg\min_{f\in {\cal F}} \frac{1}{V}\sum_{v=1}^V P_0 L(\hat{\Psi}_f(P_{n,v}))\\
&=&\arg\min_{f\in {\cal F}} \frac{1}{V}\sum_{v=1}^V d_0(\hat{\Psi}_f(P_{n,v}),\psi_0).
\end{eqnarray*}
That is, $f_{0n}$ selects the choice of candidate estimator among all $\{\hat{\Psi}_f:f\in {\cal F}\}$ whose realization on training samples $P_{n,v}$ is closest to the true function $\psi_0$ w.r.t. loss-based dissimilarity. Therefore, we refer to $f_{0n}$ as an oracle selector. Note that both $f_n$ and $f_{0n}$ are indexed by the choice of Super Learner model ${\cal F}$, so that we could also use notation $f_{n,\cal F}$ and $f_{0n,{\cal F}}$ as well.

\subsubsection{Rate of convergence}

After above definitions, in this section we will show that the Super Learner $f_n$ will converge to the oracle choice $f_{0n}$  at a rate at least faster than $O_p(n^{-1/4})$ being $n$ the number of observations. For a given class of functions ${\cal G}\subset D[0,\tau_o]$ of function of $O$, and probability measure $Q$, let  $N(\epsilon,{\cal G},L^2(Q))$ denote its covering number, defined as the minimal number of balls with radius $\epsilon$ needed to cover ${\cal G}$ in the Hilbert space $L^2(Q)$. 

\begin{theorem}
Consider the class of functions
 ${\cal G}=\{L(\hat{\Psi}_f(P_{n,v})): f\in {\cal F}\}$ and let $\alpha=\alpha({\cal F})$ be such that $\sqrt{\sup_Q\log N(\epsilon,{\cal G},L^2(Q))}\ll \epsilon^{-(1-\alpha)}$. We have that $\alpha({\cal F})\geq \alpha(d_1)$, where $\alpha(d_1)=2/(2+d_1)$.
 Then,
\[0\leq \frac{1}{V}\sum_v P_0L(\hat{\Psi}_{f_n}(P_{n,v}))-P_0L(\hat{\Psi}_{f_{0n}}(P_{n,v}))=O_P(r(n)),\]
where $r(n)=O_P(n^{-1/2})$. If, analogue to our condition that $M_2<\infty$, we assume that for a universal $M_{2,cv}$
\begin{equation}\label{M2cv}
\frac{\frac{1}{V}\sum_v P_0\left\{ L(\hat{\Psi}_{f_n}(P_{n,v}))-L(\hat{\Psi}_{f_{0n}}(P_{n,v})) \right\}^2}{\frac{1}{V}\sum_v \{P_0L(\hat{\Psi}_{f_n}(P_{n,v})-P_0L(\hat{\Psi}_{f_{0n}}(P_{n,v}) )\}}\leq M_{2,cv}
\end{equation}

then, $r(n)=O_P(n^{-1/2-\alpha({\cal F})})$.

\end{theorem}
{\bf Proof:}
We have
\begin{eqnarray*}
0&\leq &\frac{1}{V}\sum_v P_0L(\hat{\Psi}_{f_n}(P_{n,v})-P_0L(\hat{\Psi}_{f_{0n}}(P_{n,v}))\\
&=&\frac{1}{V}\sum_v -(P_{n,v}^1-P_0)\{L(\hat{\Psi}_{f_n}(P_{n,v}))-L(\hat{\Psi}_{f_{0n}}(P_{n,v}))\}\\
&&+\frac{1}{V}\sum_v P_{n,v}^1 \{ L(\hat{\Psi}_{f_n}(P_{n,v}))-L(\hat{\Psi}_{f_{0n}}(P_{n,v}))\}\\
&\leq&\frac{1}{V}\sum_v -(P_{n,v}^1-P_0)\{L(\hat{\Psi}_{f_n}(P_{n,v}))-L(\hat{\Psi}_{f_{0n}}(P_{n,v}))\}.
\end{eqnarray*}
Consider the $v$-specific empirical process term, conditional on $P_{n,v}$.
By assumption, $\hat{\Psi}_{f_n}(P_{n,v})$ and $\hat{\Psi}_{f_{0n}}(P_{n,v})$ are elements of ${\bf \Psi}$ so that these are cadlag functions with a universal bound on their sectional variation norm. More precisely, conditional on the training sample, the ensembles cover the subspace $\{\hat{\Psi}_f(P_{n,v})): f\in {\cal F}\}$, which implies a corresponding class ${\cal G}=\{L(\hat{\Psi}_f(P_{n,v})): f\in {\cal F}\}$. 
We also assumed that $L({\bf \Psi})$ is contained in cadlag functions $D[0,\tau_o]$ with a universal bound on the sectional variation norm. 
Since this is a Donsker class, it follows that the empirical process term is $O_P(n^{-1/2})$. This proves that the left-hand side is $O_P(n^{-1/2})$. 
Given (\ref{M2cv}), it follows that, for each $v$, the $L^2(P_0)$-norm of $L(\hat{\Psi}_{f_n}(P_{n,v}))-L(\hat{\Psi}_{f_{0n}}(P_{n,v}))$ is $O_P(n^{-1/4})$. As in \cite{van2017generally}, relying on the finite sample modulus of continuity bound for empirical processes indexed by a class of functions with covering number bounded in terms of $\alpha({\cal F})$ \cite{van2011local}, it follows that  the $v$-specific empirical process is $O_P(n^{-1/2-\alpha({\cal F})})$.
This proves the theorem. $\Box$

\subsection{The Conditional Super Learner}

{\bf Single step conditional Super Learner, treating the experts as known:}
The Conditional Super Learner presented in this article corresponds with a particular type of meta-learning model ${\cal F}$ consisting of the following type of functions
\[
f_{{\cal A},d}(\psi_1(x),\ldots,\psi_J(x))=\sum_{l} I_{A_l}(x) \psi_{d(l)}(x),\]
which are indexed by a partitioning ${\cal A}=\{A_l, l=1,\ldots,L)$ of $[0,\tau]$ (i.e., covariate space), and a classifier $d:\{1,\ldots,L\}\rightarrow \{1,\ldots,J\}$ for which 
$d(l)$ represents the choice of function among the $J$ functions $\{\psi_1,\ldots,\psi_J\}$: its relation with $o(x)$ in previous section is that if $x\in A_l$, then $o(x)=d(l)$, so that $o(x)$ is determined by the partitioning $\cal A$ and the classifier $d$.
In this manner, ${\cal F}=\{f_{{\cal A},d}: {\cal A},d\}$ consists of such functions for varying choices of partitioning ${\cal A}$ and corresponding classifiers $d$.
Given $\hat{\Psi}_j$, $j=1,\ldots,J$, and thereby the cross-validated data set $(Y_i,(\hat{\Psi}_j\left(P_{n,v(i)} \right):j=1,\ldots,J))$
(here $v(i)$ is the sample split for which $i$ is in the $v$-th validation sample), the cross validation selector $f_n$ corresponds now with selecting a $({\cal A}_n,d_n)$ of partitioning ${\cal A}_n$ of the covariate space and classifier $d_n$ obtained with a particular algorithm aiming to approximately minimize the cross-validated empirical risk over all ensembles $f_{{\cal A},d}(\hat{\Psi}_j(P_n));j=1,\ldots,J)$ across possible choices ${\cal A}$ and $d$ (defined by meta-learning model ${\cal F}$). This Super Learner corresponds with the simple (single iteration) conditional super-learner which treats the candidate estimators $\hat{\Psi}_j$ as fixed, not affected by $d$.

{\bf Conditional Super learner:}

For each set $A$, we can define $\hat{\Psi}_{j,A}(P_n)=\hat{\Psi}_j(P_{n,A})$, where $P_{n,A}$ is the empirical probability measure of the subsample $\{i:X_i\in A\}$. We now define the ensembles \[
f_{{\cal A},d}(\hat{\Psi}_j:j=1,\ldots,J)(P_n)=\sum_l I_{A_l}(x)\hat{\Psi}_{d(l)}(P_{n,A_l})\] indexed by a partitioning ${\cal A}$ of rectangles and classification function $d$. 
Let ${\cal F}=\{f_{{\cal A},d}: {\cal A},d\}$ for a set of possible partitioning and classifiers d. We could also index this ensemble by a single classifier function $o(x)$ by setting $A_l=\{x:o(x)=l\}$, and $d(l)=l$. The cross-validated risk for this ensemble $\hat{\Psi}_{f_{{\cal A},d}}$ of $(\hat{\Psi}_j: j=1,\ldots,J)$ is given by
\[
CV(\hat{\Psi}_{f_{{\cal A},d}},P_n)=\frac{1}{V}\sum_{v=1}^V P_{n,v}^1L(f_{{\cal A},d}(\hat{\Psi}_j:j=1,\ldots,J)(P_{n,v})).\]
Our conditional Super Learner is defined as the cross validation selector:
\[
f_{{\cal A}_n,d_n}=\arg\min_{f\in {\cal F}}CV(\hat{\Psi}_{f_{{\cal A},d}},P_n).\]
Our algorithm for determining or approximating this cross validation selector uses an iterative algorithm of the type: let $m=0$, initiate a partitioning ${\cal A}^m=\{A_l^m:l\}$, and minimize
\[
(d^{m+1},{\cal A}^{m+1})=\arg\min_{d,{\cal A}} {\cal A}\frac{1}{V}\sum_{v=1}^V P_{n,v}^1 L\left(\sum_l I_{A_l}\hat{\Psi}_{d(l)}\left(P_{n,A_l^m}\right)\right),\]
and, update ${\cal A}^{m+1}=\{A_l^{m+1}:l\}$, and iterate.

Either of the classes of functions satisfies the conditions of Theorem 3.1, if the partitionings are not too fine so that the sectional variation norm remains uniformly bounded. 
{\bf Theorem 3.1} , therefore, also applies to our $CSL$ and as such this one will converge to the oracle choice at a rate faster than $O_p(n^{-1/4})$. In practicality, though, the CSL will be controlled by the variance bias trade-off as illustrated in the next section. 

\subsection{Bias Variance Trade-off of the Super Learners}
Let \[
d_0(\hat{\Psi}_{f_n}(P_{n,v}),\hat{\Psi}_{f_{0n}}(P_{n,v}))=
P_0L(\hat{\Psi}_{f_n}(P_{n,v}))-P_0L(\hat{\Psi}_{f_{0n}}(P_{n,v})).\]
So we have
\begin{eqnarray*}
\frac{1}{V}\sum_v d_0(\hat{\Psi}_{f_n}(P_{n,v}),\psi_0)&=&
 \frac{1}{V}\sum_v d_0(\hat{\Psi}_{f_n}(P_{n,v}),\hat{\Psi}_{f_{0n}}(P_{n,v}))\\
 &&+
 \frac{1}{V}\sum_v d_0(\hat{\Psi}_{f_{0n}}(P_{n,v}),\psi_0)\\
 &\equiv& E_n({\cal F})+B_n({\cal F}).
 \end{eqnarray*}
 As the size of ${\cal F}$ increases, the "estimation-term"  $E_n({\cal F})$ will worsen, while the "bias term" $B_n({\cal F})$ will improve.  Therefore, a good choice of meta-learning model needs to carefully trade-off the variance term $E_n({\cal F})$ and bias term $B_n({\cal F})$ at the meta-learning level.
 The optimal choice of meta-learning model  ${\cal F}$ among a set of candidate meta-learning models will depend on this trade-off, and the precise trade-off depends on the candidate meta-learning models and the library estimators.
 By the previous theorem, we always have
 \[
 E_n({\cal F})=O_P(n^{-1/2-\alpha(d_1)}).\]
 Thus, the ${\cal F}$-super-learner will always achieve a rate of convergence $n^{-1/2-\alpha(d_1)}$ w.r.t. loss-based dissimilarity under the assumption that the oracle selected candidate estimator achieves at minimal that same rate of convergence.

If one considers relatively large meta-learning models, as is easily the case for our conditional super-learner, then there is a risk that one worsens the performance relative to simpler meta-learning models. Therefore, the most sensible strategy is to define a sequence of Super Learner models ${\cal F}_k$ with increasing complexity as $k$ increases, and use cross validation to select the best ${\cal F}_k$-specific super-learner among $k=1,\ldots,K$. Due  to the oracle inequality for the discrete cross validation selector this guarantees that the resulting super-learner will be asymptotically equivalent with the oracle selected super-learner among the $K$ ${\cal F}_k$-specific super-learners. 
In this manner, one is guaranteed to outperform any given ${\cal F}$-specific super-learner by including this choice of meta-learning model in our collection of Super Learners.
We will refer to this as double super-learning. In the next section we study double super-learner in some detail.

\subsection{The Double Super-Learner to learn the best Super Learner}

Let ${\cal F}_k$, $k=1,\ldots,K$, be a collection of meta-learning models. 
In  a typical setting, the size of the family ${\cal F}_k$ will be growing as $k$ increases, so that the corresponding meta-learning step is increasingly data adaptive as $k$ increases. Each family ${\cal F}_k$, defines a collection of candidate estimators $\hat{\Psi}_f(P_n)$ indexed by $f\in {\cal F}_k$, and the meta-learning step is defined by minimizing the cross-validated risk $CV(\hat{\Psi}_f,P_n)$ over all $f\in {\cal F}_k$.  Some of the families ${\cal F}_k$ might be discrete sets and linear models so that the meta learner corresponds with simply selecting the best estimator among the $J$ candidate estimators in the library and the best linear combination of these $J$ estimators, respectively.



{\bf ${\cal F}_k$-cross validation selector:}
Let
\[
f_{n,k}=\arg\min_{f\in {\cal F}_k} CV(\hat{\Psi}_f,P_n)\]
be the Super Learner defined by minimizing the cross-validated risk of $\hat{\Psi}_f$ over all $f\in {\cal F}_k$.

{\bf ${\cal F}_k$-super-learner:}
The ${\cal F}_k$-super-learner is now defined by the corresponding estimator applied to the whole sample:
\[
\hat{\Psi}_{{\cal F}_k}(P_n)=\hat{\Psi}_{f_{n,k}}(P_n),\]
or, the immediate available (not requiring rerunning $\hat{\Psi}_{f_{n,k}}$ on whole sample after having determined $f_{n,k}$): 
\[
\hat{\Psi}_{{\cal F}_k}(P_n)=\frac{1}{V}\sum_{v=1}^V \hat{\Psi}_{f_{n,k}}(P_{n,v}).\]
For notational convenience, let's also use notation $\hat{\Psi}_k^{SL}(P_n)=\hat{\Psi}_{f_{n,k}}(P_n)$.

{\bf Cross validation selector among the $K$ ${\cal F}_k$-super-learners:}
Let 
\begin{eqnarray*}
k_n&=&\arg\min_k \frac{1}{V}\sum_{v=1}^V P_{n,v}^1 L(\hat{\Psi}_k^{SL}(P_{n,v})).
\end{eqnarray*}
We also consider the oracle selector of $k$:
\begin{eqnarray*}
\tilde{k}_n&=&\arg\min_k \frac{1}{V}\sum_{v=1}^V P_0L(\hat{\Psi}^{SL}_k(P_{n,v})).\\
\end{eqnarray*}

{\bf Proposed double super-learner:}
Then, our double super-learner and proposed estimator is defined by 
\[
\hat{\Psi}(P_n)=\hat{\Psi}^{SL}_{k_n}(P_n).\]

{\bf Inner $V_1$-fold cross validation within training sample $P_{n,v}$:} 
Note that for computing $k_n$, and thereby $\hat{\Psi}^{SL}_k(P_{n,v})$ we need to apply a cross validation scheme to the training sample $P_{n,v}$ itself. 
Consider a $V_1$-fold cross validation scheme where for each sample split $v_1=1,\ldots,V_1$ of a sample $P_{n,v}$   into validation sample $P_{n,v,v_1}^1$ and complementary training sample $P_{n,v,v_1}$. We can then apply our definition of $\hat{\Psi}^{SL}_k$ above to the data $P_{n,v}$ using this inner $V_1$-fold cross validation scheme (i.e., replace $P_n$ by $P_{n,v}$ and replace $P_{n,v}$, by $P_{n,v,v_1}$).

{\bf Double cross validation:} If we apply the inner $V_1$-fold cross validation scheme to a training sample $P_{n,v}$ from an outer cross validation scheme, then this $V_1$-cross validation splits $P_{n,v}$ into $P_{n,v,v_1}$ and $P_{n,v,v_1}^1$, $v_1=1,\ldots,V_1$. So $(V,V_1)$-double cross validation maps $P_n$ into $V\times V_1$-sample splits $(P_{n,v,v_1},P_{n,v,v_1}^1:v_1=1,\ldots,V_1, v=1\ldots,V)$.

\subsection{Oracle inequality for double super-learner showing asymptotic equivalence with super-learner using oracle choice of meta-learning model}
By \cite{van2006oracle} we have, for any $\delta>0$, and constant $C(M_1,M_2,\delta)=2(1+\delta)^2(M_1/3+M_2/\delta)$,
\[ \frac{E_0 \frac{1}{V}\sum_v d_0(\hat{\Psi}^{SL}_{k_n}(P_{n,v}),\psi_0)}{E_0\\min_k \frac{1}{V}\sum_v d_0(\hat{\Psi}^{SL}_k(P_{n,v}),\psi_0)+C(M_1,M_2,\delta)\frac{\log K_n}{np}}
\leq (1+\delta) \]
If  the number of Super Learners does not grow faster than a polynomial in sample size $n$, i.e, $K_n\ll n^p$ for some $p$, and 
\[
\frac{n^{-1}\log n}{E_0\min_k \frac{1}{V}\sum_v d_0(\hat{\Psi}^{SL}_k(P_{n,v}),\psi_0)}\rightarrow 0,\]
then the  double super-learner $\hat{\Psi}^{SL}_{k_n}$ is asymptotic equivalent with the oracle selected super-learner $\hat{\Psi}_{\tilde{k}_n}$:
\[
\frac{E_0 \frac{1}{V}\sum_v d_0(\hat{\Psi}^{SL}_{k_n}(P_{n,v}),\psi_0)}{E_0\\min_k \frac{1}{V}\sum_v d_0(\hat{\Psi}^{SL}_k(P_{n,v}),\psi_0)}
\rightarrow 1\mbox{ as $n\rightarrow\infty$.}
\]

\subsection{A finite sample oracle inequality for an $\epsilon$-net double super-learner}
The above results show that a double super-learner will be asymptotically equivalent with choosing the best meta-learning algorithm among the sequence of Super Learners, and that, this oracle selected Super Learner approximates the oracle choice in its meta-learning model at a rate depending on its covering number, but either way, at least as fast as $n^{-1/4}$. In this subsection, we show a stronger result demonstrating that the double super-learner is able to optimally trade-off bias and variance across all possible ensembles across the meta-learning models. However, this requires generating a more refined set of meta-learning models, namely, for each meta-learning model, we define an  $\epsilon$-net (like sieve of increasing complexity) representing a resolution.  

Let ${\cal F}_k=\{f_{k,\alpha}:\alpha\in {\cal E}_k\}$ for some parametrization $\alpha\rightarrow f_{k,\alpha}$ with the $\alpha$-parameter varying over a set ${\cal E}_k$.
 
For each $k$, let ${\cal E}_k(\epsilon)=\{\alpha_{k,j}: j=1,\ldots,N_k(\epsilon)\}\subset {\cal E}_k$ be a finite set of values in ${\cal E}_k$ so that the parameter space of $\{\hat{\Psi}_{\alpha}(P_n):\alpha\in {\cal E}_k(\epsilon)\}$
represents an $\epsilon$-net of $\{\hat{\Psi}_{\alpha}(P_n):\alpha\in {\cal E}_k\}$ w.r.t. some dissimilarity (chosen to dominate or be equivalent with the loss-based dissimilarity). Let $N_k(\epsilon)$ the number of elements in ${\cal E}_k(\epsilon)$.

{\bf $(k,\epsilon)$-th super-learner:}
For each $k,\epsilon$, given a sample $P_n$, let
\begin{equation}\label{kthMLt}
\alpha_{k,\epsilon}(P_n)\equiv \arg\min_{\alpha\in {\cal E}_k(\epsilon)}\frac{1}{V_1}\sum_{v_1=1}^{V_1} P_{n,v_1}^1 L(\hat{\Psi}_{k,\alpha}(P_{n,v_1}) ).\end{equation}
Then, $\hat{\Psi}^{SL}_{k,\epsilon}:{\cal M}_{NP}\rightarrow {\bf \Psi}$ defined by \[
\hat{\Psi}^{SL}_{k,\epsilon}(P_n)\equiv \hat{\Psi}_{k,\alpha_{k,\epsilon}(P_n)}(P_n)\]
 is a super-learner based on the $(k,\epsilon)$-th meta-learning algorithm defined by (\ref{kthMLt}), i.e., by minimizing the cross-validated risk over a $\epsilon$-net of the $k$-th specific family of the library of candidate estimators $(\hat{\Psi}_j:j=1,\ldots,J)$.

{\bf Set of candidate super-learners:}
Now, $\hat{\Psi}^{SL}_{k,\epsilon}:{\cal M}_{NP}\rightarrow {\bf \Psi}$, $k=1,\ldots,K$, and $\epsilon$ varying over a grid, represents a set of candidate estimators, each one of them being a super-learner based on $(k,\epsilon)$-specific meta-learning algorithm defined by minimizing the cross-validated risk over a parametric family. 

{\bf cross validation selector of choice of meta learning, including resolution $\epsilon$:}
Let 
\begin{eqnarray*}
(k_n,\epsilon_n)&=&\arg\min_{k,\epsilon} \frac{1}{V}\sum_{v=1}^V P_{n,v}^1 L(\hat{\Psi}^{SL}_{k,\epsilon}(P_{n,v}))\\
&=&\arg\min_{k,\epsilon} \frac{1}{V}\sum_{v=1}^V P_{n,v}^1 L(\hat{\Psi}_{k,\alpha_{k,\epsilon}(P_{n,v})}(P_{n,v})),
\end{eqnarray*}
where
\[
\alpha_{k,\epsilon}(P_{n,v})=\arg\min_{\alpha\in {\cal E}_k(\epsilon)}\frac{1}{V_1}\sum_{v_1=1}^{V_1} P_{n,v,v_1}^1 L(\hat{\Psi}_{k,\alpha}(P_{n,v,v_1}) ).
\]

{\bf Proposed double  super-learner for theoretical analysis:}
Then, our final super-learner and proposed estimator is defined by 
\[
\hat{\Psi}(P_n)=\hat{\Psi}^{SL}_{k_n,\epsilon_n}(P_n).\]

Let $K_n$ be the number of values of $(\epsilon,k)$ over which we minimize in definition of $(k_n,\epsilon_n)$. 
We have the following theorem.
\begin{theorem}\label{theoremepsnet}
Suppose 
\begin{eqnarray}
Ed_0(\hat{\Psi}_{k,\alpha_{k,\epsilon}(P_n)}(P_{n,v}),\psi_0) &\leq&
E \frac{1}{V_1}\sum_{v_1} d_0(\hat{\Psi}_{k,\alpha_{k,\epsilon}(P_{n,v})}(P_{n,v,v_1}),\psi_0).
\label{aa}
\end{eqnarray}
Then, we have
\[
\begin{array}{l}
E \frac{1}{V}\sum_v d_0(\hat{\Psi}_{k_n,\epsilon_n}^{SL}(P_{n,v}),\psi_0)\leq (1+\delta)^2\\
\min_{k,\epsilon} \frac{1}{V}\sum_v \left \{ 
E  \min_{\alpha\in {\cal E}_k(\epsilon)} \frac{1}{V_1}\sum_{v_1} d_0(\hat{\Psi}_{k,\alpha}(P_{n,v,v_1}),\psi_0)+C(M_1,M_2,\delta) \frac{\log N_k(\epsilon)}{n(1-p)p_1}\right \}\\
\hfill +C(M_1,M_2,\delta)\log K_n /(np).
\end{array}
\]
\end{theorem}
{\bf Proof:}
As shown in previous section, regarding the cross validation selector $(k_n,\epsilon_n)$ of $(k,\epsilon)$ we have
\begin{eqnarray}
E\frac{1}{V}\sum_v d_0(\hat{\Psi}_{k_n,\epsilon_n}^{SL}(P_{n,v}),\psi_0)&\leq & (1+\delta)E \min_{\epsilon,k}\frac{1}{V}\sum_v  d_0(\hat{\Psi}_{k,\epsilon}^{SL}(P_{n,v}),\psi_0)\nonumber \\
&& +C(M_1,M_2,\delta)\frac{\log K_n}{np} \label{firstoracle}
\end{eqnarray}
Recall $\hat{\Psi}_{k,\epsilon}^{SL}(P_{n,v})=\hat{\Psi}_{k,\alpha_{k,\epsilon}(P_{n,v})}(P_{n,v})$.
Regarding the cross validation selector $\alpha_{k,\epsilon}(P_{n,v})$ we have
\begin{eqnarray}
E\frac{1}{V_1}\sum_{v_1} d_0(\hat{\Psi}_{k,\alpha_{k,\epsilon}(P_{n,v})}(P_{n,v,v_1}),\psi_0)&\leq& (1+\delta)E \min_{\alpha\in {\cal E}_k(\epsilon)} \frac{1}{V_1}\sum_{v_1} d_0(\hat{\Psi}_{k,\alpha}(P_{n,v,v_1}),\psi_0) \nonumber \\
&&+C(M_1,M_2,\delta) \frac{\log N_k(\epsilon)}{n(1-p)p_1}. \label{secondoracle} 
\end{eqnarray}
Suppose (\ref{aa}) holds.
Then, first term on right-hand side of (\ref{firstoracle}) is bounded by 
\[
(1+\delta) \min_{k,\epsilon} \frac{1}{V}\sum_v E \frac{1}{V_1}\sum_{v_1} d_0(\hat{\Psi}_{k,\alpha_{k,\epsilon}(P_{n,v})}(P_{n,v,v_1}),\psi_0).\]
Using the second oracle inequality (\ref{secondoracle}) yields that we can bound this by $(1+\delta)$ times
\[
\min_{k,\epsilon} \frac{1}{V}\sum_v \left \{ 
(1+\delta) E  \min_{\alpha\in {\cal E}_k(\epsilon)} \frac{1}{V_1}\sum_{v_1} d_0(\hat{\Psi}_{k,\alpha}(P_{n,v,v_1}),\psi_0)+C(M_1,M_2,\delta) \frac{\log N_k(\epsilon)}{n(1-p)p_1}\right \}.
\]
This proves the stated bound.
$\Box$

\paragraph{Discussion of Theorem \ref{theoremepsnet}}
Consider the displayed inequality in Theorem \ref{theoremepsnet}.
Note that the last term is negligible as long as $K=K_n$ is not growing faster than polynomial in $n$.
So it is all about the first term. For each $k$, view $\hat{\Psi}_{k,\alpha}(P_{n,v,v_1})$ as a $k$-specific data adaptive model with parameter $\alpha\in {\cal E}_k$. First consider the case that $\psi_0$ is contained in one of these $k$-specific models with probability tending to 1 as $n\rightarrow\infty$. For  the $k$-specific  data adaptive model,  $\min_{\alpha\in {\cal E}_k(\epsilon)}$ denotes a bias term while the covering number term  $\log N_k(\epsilon)/(n(1-p)p_1)$ represents a variance term, and it is known that optimizing over $\epsilon$ the sum of these two terms optimally trades off bias and variance (up till constant), resulting in a minimax rate of convergence for this $k$-specific data adaptive model \cite{van2003unified} 

The outer $\min_k$ shows that we achieve the minimax rate of convergence corresponding with the smallest of the data adaptive models that contains the true $\psi_0$. So if one of the data adaptive models contains the true $\psi_0$ then we would achieve the minimax adaptive rate of convergence. For most choices of $k$-specific families, one does not expect that that $\psi_0$ is contained by any of them, so that the bias term is not just driven by the $\epsilon$-resolution, but also by the bias of the overall $k$-specific model. Either way, the leading term in this oracle inequality is still properly trading of the actual bias with the variance term.

\section{Practical Aspects: Initialization, Collapsing and Soft CSL}

We initialize CSL by picking the type of experts (e.g trees, linear models, etc) and a random subset of the data to fit each expert to introduce diversity. Please note that as its counterpart (K-means), CSL can get stuck in local minimum. Therefore we run the algorithm a few times in our experiments. We use a validation set (different than the test set) to select the best solution. While running CSL it is often the case that some of the experts will collapse (e.g model selection). This happens when the range of values predicted by $o(x)$ is less than $K$ or similarly when the size of $\{x: o(x)=k\}$ becomes very small. In these cases we re-adjust the size of $K$ as the algorithm runs. 

Finally, please note that we can use the step when we are fitting the experts to introduce regularization. If the oracle $o(x)$ also estimates probability of an observation belonging to a model  (e.g. logistic regression), then we can use $p({0}(\bf{x}) = k,\bf{x})$ to introduce similarity among the experts when we are fitting them, specially around the boundaries defined by the oracle. 

\section{Simulations}

In this section we describe our experiments. We consider the dataset of all regression problems from the Penn Machine Learning Benchmarks \cite{olson2017pmlb}. In the first set of experiments, dataset where the number of observations was between 200 and 500000 ($N=84$) were considered. In the second set of experiments, we selected a subset of those with at least 2000 points ($N=19$). We report $R^2$ as the performance metric. In the first experiments, datasets were split in $80\%$ training and $20\%$ testing sets. For the second experiment, we need a validation set, therefore data was split in $70\%/ 15\%/ 15\%$ for train/validation/test.

In all our experiments we use the following set of base algorithms or experts:

 \begin{enumerate}
\item Ridge: alphas = [1e-4, 1e-3, 1e-2, 1e-1, 1, 2, 4, 8, 16, 32, 64, 132]
\item ElasticNet: l1 ratio = 0 (Lasso) 
\item ElasticNet: l1 ratio = 0.5
\item Decision Tree: max depth = 4
\item Decision Tree: max depth = 5
\item Decision Tree: max depth = 6
\end{enumerate}

Although the CSL can be used with more complex models as experts, using simple linear or tree models have two appealing: protects the algorithm for overfitting and has implications for its interpretability as discussed below. 

\begin{figure}[t!]
 \includegraphics[width=0.6\textwidth]{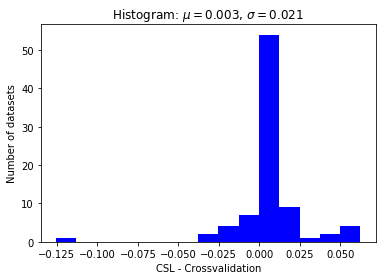}
\centering
\caption{One shot CSL vs cross validation}
\label{Figure2}
\end{figure}

\subsection{Single step CSL and its comparisson to cross validation}

First, we wanted to evaluate empirically how CSL (${\cal F}_3$ class of meta learners in the theoretical section) performs compared to the naive strategy of using the expert model selected through cross validation for all the points (${\cal F}_1$ class in the theoretical section). For this we used the library of experts described above and a decision tree algorithm with $max\_depth=1$ as our oracle. No partition was allowed if the terminal node did not have more than $2\%$ of observations belonging to it. This decision tree algorithm was selected as the oracle because if not partition is performed, then the minimization of equation \ref{eqn1} results in selecting the model that minimizes the cross validation error (equal to ${\cal F}_1$ class ). As such, our CSL in this case includes the possibility of just using cross validation and it should perform, on average, better than using cross validation to select one model. Please note that in this section the experts are obtained on all the training data and only one iteration is allowed for the meta. The empirical evaluation of this CSL was compared to the performance obtained if we use cross validation to select the best model from one of the sixth algorithms mentioned above. 
  
 Figure \ref{Figure2} shows the results obtained when we compared $R^2$ for both CSL and cross validation. In $77.5\%$ of the datasets CSL had at least the same performance as cross validation and in $20 \%$ its $R^2$ was bigger by at least 1. Although these results show that CSL improves over naive cross validation, overfitting can happen and extra attention needs to be paid. In fact, we obtained an outlier where cross validation outperformed the CSL by 0.125. This problem is the synthetic dataset $658-fri-c3-250-25$ from the Friedman's regression datasets \cite{olson2017pmlb,friedman1991multivariate}. The same is a hard problem where algorithms are prompt to overfit since training only contains 200 points with 25 explanatory variables, several of them correlated to each other. Additionally, for those datasets where CSL was better, on average they had 3805 data points compared to 2368 in those datasets where cross validation did better. Therefore, as a general rule, CSL needs more than 2500 data points to perform better than cross validation. These empirical results corroborate our suggestions on the theoretical session and the need for a two tier Super Learner algorithm where the type of meta strategy, or class of functions as defined here, is also selected. 
 
 In the next section, we will compare a full CSL  to stacking. ${\cal F}_2$ vs ${\cal F}_3$ meta learning strategies in our theoretical discussion.

\subsection{CSL versus Stacking} 
\begin{table}
 \begin{tabular}{|l|c|c|c|c|}%
    \hline
    \bfseries Dataset & \bfseries CSL\_mean & \bfseries Stack\_mean & \bfseries Diff & \bfseries Test result
    \csvreader[head to column names]{grade.csv}{}
    {\\\hline \dataset & \islmean & \stackmean & \diffmean & \test}
    \tabularnewline\hline
    \end{tabular}
    \caption{Results comparing CSL and stacking on 19 regression datasets. In one dataset stacking is significantly better than CSL. In 8 problems both algorithms are statistically the same. In 11 problem CSL is better.}
    \label{table_stacking}
\end{table}    

{\bf Stacking}  \cite{wolpert1992stacked, breiman1996stacked,SmythWolpert1997} is
a general procedure where a learner is trained to combine the individual learners. The base level models are trained on the original training set, then a meta-model is trained on the outputs of the base level model as features. The base level often consists of different learning algorithms. In our experiments, we use the same 6 base models defined in previous section. As meta-model in stacking we use a linear regression model which corresponds to ${\cal F}_2$.  

{\bf CSL}. In this experiment, we use the same set of experts as in previous section and we use a two layer feed-forward neural network as the oracle. The oracle was written in PyTorch and fitted with Adam optimizer with learning rates of 0.15. The number of epochs at each iteration was a function of the sample size ($ \frac{3000}{\log(N)^2}$). The hidden layer was also set as a function of the sample size ($\min(2\log(N), 150)$). After the first linear layer and before the Relu activation function, batch normalization was used. A dropout layer with $p=0.2$ is used before the second linear layer.

In Table \ref{table_stacking} we show results from comparing CSL and stacking on regression datasets. For each dataset we run each algorithm 10 times by spliting training, validation and testing sets using different seed. \emph{CSL\_mean} shows the mean $R^2$ over all experiments. Similarly, \emph{Stack\_mean} shows the mean $R^2$ of the stacking experiments. \emph{Diff} show the difference between \emph{CSL\_mean} and \emph{Stack\_mean}. Column \emph{Test results} shows whether a t-test found the difference in mean to be significant. There were 19 datasets in this experiment. For one dataset stacking is significantly better than CSL. In 8 problems both algorithms are statistically the same. In 11 problem CSL is better. In the experiment for dataset \emph{574\_house\_16H}, one of the runs produces an outlier which is responsible for the mean difference.

\subsection{CSL on hierarchical synthetic data}

Hierarchical models are extremely important in fields like Medicine \cite{normand1997statistical}. CSL, due to its architecture, seems specially suited to analyze hierachical data. In this section, we evaluate the performance of CSL on hierarchical synthetic problems. Specifically, we wanted to investigate if:

\begin{enumerate}
\item CSL can discover hierarchical structures.
\item Compare its performance in these type of problems to gradient boosting and random forest.
\end{enumerate}

Given a real dataset $\{(x_i, y_i)\}^N_1$ from the Penn Machine Learning Benchmarks, we generate syntetic data by using the observations from the covariates but generating new labels. Here is how we generated the new ($\tilde{y}$): 
\begin{enumerate}
\item Sample 70\% of the data.
\item Use the covariates to find $K=3$ clusters using K-means algorithm.
\item Fit a 2-layer neural network using the cluster id as a label.
\item Using the neural network, predict a cluster id ($l_i$) on each of the original observations, creating the dataset $\{(x_i, y_i, l_i)\}^N_1$.
\item Fit each subset $\{(x_i, y_i) \text{ such that } l_i=k \}$ for $k \in \{1, 2, 3\}$ to a regression model (Ridge). Use the prediction from the regression model as the new synthetic label ($\tilde{y}$).
\end{enumerate}

\begin{table}
 \begin{tabular}{|l|c|c|c|c|}%
    \hline
    \bfseries Dataset & \bfseries CSL & \bfseries Base Expert & \bfseries RF & \bfseries GBM \\
    \hline
    564\_fried & \textbf{0.97} (0.02) & 0.45 (0.17) & 0.82 (0.04) &  0.82 (0.06) \\
    574\_house\_16H & \textbf{0.98} (0.01) & 0.84 (0.04) & 0.87 (0.01) & 0.93 (0.01)\\
    294\_satellite\_image & \textbf{0.99} (0.01) & 0.88 (<0.01) & 0.98 (0.01) & 0.98 (<0.01)\\
    218\_house\_8L & \textbf{0.97} (0.02) & 0.75 (0.03) & 0.94 (0.01) & 0.93 (0.01)
    \tabularnewline\hline
    \end{tabular}
    \caption{Results on running CSL on synthetic data. Comparison are made with random forest, gradient boosting and the best preforming base expert. For each dataset the we generate 10 synthetic problems. For each method, the mean test $R^2$ and standard deviation is shown.}
    \label{table_synt}
\end{table}

For the experiments in Table \ref{table_synt}, gradient boosting was ran with the following parameters: min\_child\_weight=50, learning rate = 0.1, colsample\_bytree= 0.3, max\_depth= 15, subsample=0.8, and with 500 trees. For random forest, we used 1000 trees, max\_features='sqrt' and we found max\_depth with cross validation for each problem. The problems are a subset of the problems in Table \ref{table_stacking}, where the $R^2$ of the base expert is lower than 0.9. The CLS alrorithm was initialized using 11 linear models as based experts. It turns out that by using trees together with linear models on these problems the algorithm sometimes would get stuck in suboptimal local minima. 

Table \ref{table_synt} shows the results ($R^2$) of CSL on synthetic data compared to random forest, gradient boosting or the best base expert (selected using cross validation). As it can be seen, CSL significantly outperforms all other algorithms as expected.

\subsection{Implications on interpretability} 

In many high-stake domains like medicine and the law, interpretability of machine learning algorithms is highly desirable since errors can have dire consequences \cite{valdes2016mediboost, louzada2016classification, cabitza2018machine}. This fact has led to a renew interest for the development of interpretable models. Interpretability, however, can only be judge relative to the field of application as it is in the eyes of the beholders. In fact, there is a considerable disagreement on what the concept means and how to measure it \cite{lipton2016mythos,lipton2017doctor, doshi2017towards}. Different algorithms afford different degrees of interpretability and even  black boxes can be investigated to gain some intuition on how predictions are being made. For instance, variable importance or distillation can be used to interpret neural networks \cite{tan2018transparent}. This level of interpretability might be enough for applications that do not impose high risk. In other applications (e.g medicine), the need to understand the models globally rises \cite{valdes2016mediboost}. Without attempting to formally quantify and define interpretability here, we will illustrate below how the $CSL$ results in models that are highly transparent. 

{\bf Predicting house prices.} To illustrate how CSL can be use as an interpretable algorithm we use a dataset of house rental prices in New York City. We have 4 input variables: latitude, longitude, number of bedrooms and number of bathrooms. Two make it really simple to visualize and interpet, the oracle was given two of the variables: latitude and longitude. The CSL model found a solution in which the oracle parition the space of latitude and longitude in 3 regions (see top of Figure \ref{Figure1}) and for each region a tree of depth 5 predicts the house prices. This simple solution get an $R^2$ of 0.68. As a comparison, the best single model of a tree of depth 5 has an $R^2$ of 0.62 and a random forest with 500 trees (of depth 9) has an $R^2$ of 0.72. To find the best random forest we did grid search on the number of variables and the depth of the trees.

The simple solution of a tree of depth 5 is interpretable since a tree of depth 5 can be easily examined. Also, 3 trees of depth 5 can be easily examined as well as the 2 dimensional space where the oracle split the restricted input regions. On the other hand, the random forest with 500 trees and unrestricted depth can not be called interpretable. 

\subsection{Connection to other algorithms}

Our CSL is related to different algorithms and highlighting them here can give us additional intuitions about its performance, what problems are best suited for it and how to improve it. First, please note that $o(x)$ partitions the space in $K$ regions or subsets $\{\mathcal{R}\}_1^K$ where the models $\{F_{k}(x)\}_1^K$ are used for prediction; ergo establishing the connection between meta learning and generalized partitioning machines. Different from recursive algorithms like CART, MediBoost or the Additive Tree, CSL partitions defined by the oracle can have complex forms and are not forced to be perpendicular to the covariates. Additionally, $CSL(x)$ also generalizes the strategy of using cross validation to select the best model. Please note that if in equation \ref{eqn1} we force $o(\bf{x}) = c$ where $c$ is a constant $\in \{1..K\}$ then the solution to \ref{eqn4}, $\hat{o}(\bf{x})$, just selects the model that minimizes the cross validation error. As such, using cross validation to select the best model is the simplest case of $CSL(x)$ where the meta learner predicts a constant regardless of the covariate. CSL can also be thought of a generalization of the K-means algorithm. If the expert models are constant, and the oracle has infinity capacity to always be able to assign each observation to the best mean, then CSL becomes the K-mean algorithm. Finally, as shown above, due to its archictecure, CSL is a non parametric hierarchical algorithm and performs quite well for this type of problems. 

\section{Conclusions}

In this work we introduced the CSL algorithm. We proved theoretically and empirically how we can extend the idea of meta learning and develop an algorithm that outperforms the naive use of cross validation to select the best model.We proved that the CSL has a rate of convergance faster than $O_p(n^{-1/4})$. More over, we have obtained very interesting and practical results. For instance, CSL outperformed stacking in the datasets analyzed. Additionally, it significantly outperformed Random Forests or Gradient Boosting in the analysis of Hierarchical Data. Finally, its connection to interpretability and other algorithms were highlighted to deepen our understanding of its performance. As such, the CSL is an algorithm suited for the analysis of medical datasets where hierarchical models and intepretability are of paramount importance. 

\section{Acknowledgments}

Research reported in this publication was supported by the National Institute Of Biomedical Imaging And Bioengineering of the National Institutes of Health under Award Number K08EB026500 and by the National Institute of Allergy and Infectious Diseases under Award Number 5R01AI074345-09. The content is solely the responsibility of the authors and does not necessarily represent the official views of the National Institutes of Health.Finally, we would also like to thank Dr Charles McCullow for initially suggesting to investigate this topic. 

\bibliographystyle{tfs}
\bibliography{references}

\begin{thebibliography}{10}
\providecommand{\MR}{\relax\unskip\space MR }
\providecommand{\url}[1]{\normalfont{#1}}
\providecommand{\urlprefix}{Available at }

\bibitem{breiman1996bagging}
L. Breiman, \emph{Bagging predictors}, Machine learning 24 (1996), pp.
  123--140.

\bibitem{breiman1996stacked}
L. Breiman, \emph{Stacked regressions}, Machine learning 24 (1996), pp. 49--64.

\bibitem{cabitza2018machine}
F. Cabitza and G. Banfi, \emph{Machine learning in laboratory medicine: waiting
  for the flood?}, Clinical Chemistry and Laboratory Medicine 56 (2018), pp.
  516--524.

\bibitem{doshi2017towards}
F. Doshi-Velez and B. Kim, \emph{Towards a rigorous science of interpretable
  machine learning}, arXiv preprint arXiv:1702.08608  (2017).

\bibitem{efron1983estimating}
B. Efron, \emph{Estimating the error rate of a prediction rule: improvement on
  cross-validation}, Journal of the American statistical association 78 (1983),
  pp. 316--331.

\bibitem{freund1996experiments}
Y. Freund, R.E. Schapire, \emph{et~al.}, \emph{Experiments with a new boosting
  algorithm}, in \emph{icml}, Vol.~96. Citeseer, 1996, pp. 148--156.

\bibitem{friedman2001greedy}
J.H. Friedman, \emph{Greedy function approximation: a gradient boosting
  machine}, Annals of statistics  (2001), pp. 1189--1232.

\bibitem{friedman1991multivariate}
J.H. Friedman, \emph{et~al.}, \emph{Multivariate adaptive regression splines},
  The annals of statistics 19 (1991), pp. 1--67.

\bibitem{gill1995inefficient}
R.D. Gill, M.J. Laan, and J.A. Wellner, \emph{Inefficient estimators of the
  bivariate survival function for three models}, in \emph{Annales de l'IHP
  Probabilit{\'e}s et statistiques}, Vol.~31. 1995, pp. 545--597.

\bibitem{leblanc1996combining}
M. LeBlanc and R. Tibshirani, \emph{Combining estimates in regression and
  classification}, Journal of the American Statistical Association 91 (1996),
  pp. 1641--1650.

\bibitem{lipton2016mythos}
Z.C. Lipton, \emph{The mythos of model interpretability}, arXiv preprint
  arXiv:1606.03490  (2016).

\bibitem{lipton2017doctor}
Z.C. Lipton, \emph{The Doctor Just Won't Accept That!}, in \emph{Proc. of Symp.
  of Interpretable Machine Learning at the Intl. Conf. on Neural Inf.
  Processing Sys. (NIPS)}. 2017, pp. 1 -- 3.

\bibitem{louzada2016classification}
F. Louzada, A. Ara, and G.B. Fernandes, \emph{Classification methods applied to
  credit scoring: Systematic review and overall comparison}, Surveys in
  Operations Research and Management Science 21 (2016), pp. 117--134.

\bibitem{neuhaus1971weak}
G. Neuhaus, \emph{et~al.}, \emph{On weak convergence of stochastic processes
  with multidimensional time parameter}, The Annals of Mathematical Statistics
  42 (1971), pp. 1285--1295.

\bibitem{normand1997statistical}
S.L.T. Normand, M.E. Glickman, and C.A. Gatsonis, \emph{Statistical methods for
  profiling providers of medical care: issues and applications}, Journal of the
  American Statistical Association 92 (1997), pp. 803--814.

\bibitem{olson2017pmlb}
R.S. Olson, W. La~Cava, P. Orzechowski, R.J. Urbanowicz, and J.H. Moore,
  \emph{Pmlb: a large benchmark suite for machine learning evaluation and
  comparison}, BioData mining 10 (2017), p.~36.

\bibitem{SmythWolpert1997}
P. Smyth and D. Wolpert, \emph{Stacked Density Estimation}, in
  \emph{Proceedings of the 10th International Conference on Neural Information
  Processing Systems}, Cambridge, MA, USA. MIT Press, NIPS'97, 1997, pp.
  668--674. \urlprefix\url{http://dl.acm.org/citation.cfm?id=3008904.3008999}.

\bibitem{tan2018transparent}
S. Tan, R. Caruana, G. Hooker, and A. Gordo, \emph{Transparent model
  distillation}, arXiv preprint arXiv:1801.08640  (2018).

\bibitem{valdes2016mediboost}
G. Valdes, J.M. Luna, E. Eaton, C.B. Simone, \emph{et~al.}, \emph{Mediboost: a
  patient stratification tool for interpretable decision making in the era of
  precision medicine}, Scientific Reports 6 (2016).

\bibitem{van2017generally}
M. van~der  Laan, \emph{A generally efficient targeted minimum loss based
  estimator based on the highly adaptive lasso}, The international journal of
  biostatistics 13 (2017).

\bibitem{van2003unified}
M.J. Van Der~Laan and S. Dudoit, \emph{Unified cross-validation methodology for
  selection among estimators and a general cross-validated adaptive epsilon-net
  estimator: Finite sample oracle inequalities and examples}  (2003).

\bibitem{van2007super}
M.J. Van~der  Laan, E.C. Polley, and A.E. Hubbard, \emph{Super learner},
  Statistical applications in genetics and molecular biology 6 (2007).

\bibitem{van2011local}
A. Van Der~Vaart and J.A. Wellner, \emph{A local maximal inequality under
  uniform entropy}, Electronic Journal of Statistics 5 (2011), p. 192.

\bibitem{van2006oracle}
A.W. Van~der  Vaart, S. Dudoit, and M.J. van~der  Laan, \emph{Oracle
  inequalities for multi-fold cross validation}, Statistics \& Decisions 24
  (2006), pp. 351--371.

\bibitem{wolpert1992stacked}
D.H. Wolpert, \emph{Stacked generalization}, Neural networks 5 (1992), pp.
  241--259.

\end{thebibliography}

\end{document}